\documentclass[10pt,twocolumn,letterpaper]{article}

\usepackage[table]{xcolor}
\usepackage{cvpr}
\usepackage{times}
\usepackage{epsfig}
\usepackage{graphicx}
\usepackage{amsmath}
\usepackage{amssymb}

\usepackage{url}
\usepackage{amsthm}
\usepackage{verbatim}
\usepackage{stmaryrd}
\usepackage{color}
\usepackage{newclude}
\usepackage{algorithm}
\usepackage{algorithmic}

\newcommand{\KL}{\operatorname{KL}}
\newcommand{\sign}{\operatorname{sign}}
\newcommand{\er}{\operatorname{er}}
\newcommand{\eer}{\widehat{\er}}

\newcommand{\Erf}{\operatorname{erf}}
\newcommand{\Id}{\operatorname{\textit{Id}}}

\newcommand{\pushcode}[1][1]{\hskip\dimexpr#1\algorithmicindent\relax}
\newtheorem{theorem}{Theorem}
 \cvprfinalcopy 

\def\cvprPaperID{1964} 

\ifcvprfinal\fi
\begin{document}

\title{Curriculum Learning of Multiple Tasks}

\author{Anastasia Pentina\\
\and
Viktoriia Sharmanska\\
IST Austria, Klosterneuburg, Austria \\
{\tt\small $\{$apentina, vsharman, chl$\}$@ist.ac.at}
\and
Christoph H. Lampert
}

\maketitle
\begin{abstract}
Sharing information between multiple tasks enables algorithms to achieve good generalization performance even from small amounts of training data. 
However, in a realistic scenario of multi-task learning not all tasks are equally related to each other, 
hence it could be advantageous to transfer information only between the most related tasks.

In this work we propose an approach that processes multiple tasks in a sequence with sharing between subsequent tasks instead of solving all tasks jointly. 
Subsequently, we address the question of curriculum learning of tasks, i.e.  finding the best order of tasks to be learned. 
Our approach is based on a generalization bound criterion for choosing the task order that optimizes the average expected classification performance over all tasks.
Our experimental results show that learning multiple related tasks sequentially can be more effective than learning them jointly, the order in which tasks are being solved affects the overall performance, and that our model is able to automatically discover the favourable order of tasks.
\end{abstract}
\section{Introduction}

Multi-task learning~\cite{caruana1997multitask} studies the problem of solving several prediction tasks.
While traditional machine learning algorithms can be applied to solve each task independently, they usually need significant amounts of labelled data to achieve generalization of reasonable quality.
However, in many cases it is expensive and time consuming to annotate large amounts of data, especially in computer vision applications such as object categorization.
%
%
An alternative approach is to share information between several related learning tasks and this has been shown experimentally to allow better generalization from fewer training points per task~\cite{survey}.
%

%
In this work we focus on the parameter transfer approach to multi-task learning that rests on the idea that models corresponding to related tasks are similar to each other in terms of their parameter representations.
We concentrate on the case of linear predictors and assume that similarity between the models is measured by the Euclidean distance between the corresponding weight vectors~\cite{Yang}.
%
%
%
In a multi-task setting this idea was introduced by Evgeniou and Pontil in~\cite{Evgeniou}.
There the authors propose an SVM-based algorithm that enforces the weight vectors corresponding to different tasks to lie close to some common prototype, and they show its effectiveness on several datasets.
%
However, this algorithm treats all the tasks symmetrically, which might not be optimal in more realistic scenarios. 
There might be some outlier tasks or groups of tasks such that there is no similarity between the tasks from different groups. 
Hence, more flexible models are needed that are able to exploit the structure underlying tasks relations and avoid negative consequences of transferring information between unrelated tasks.
\begin{figure}[t]\centering
\includegraphics[width=\columnwidth]{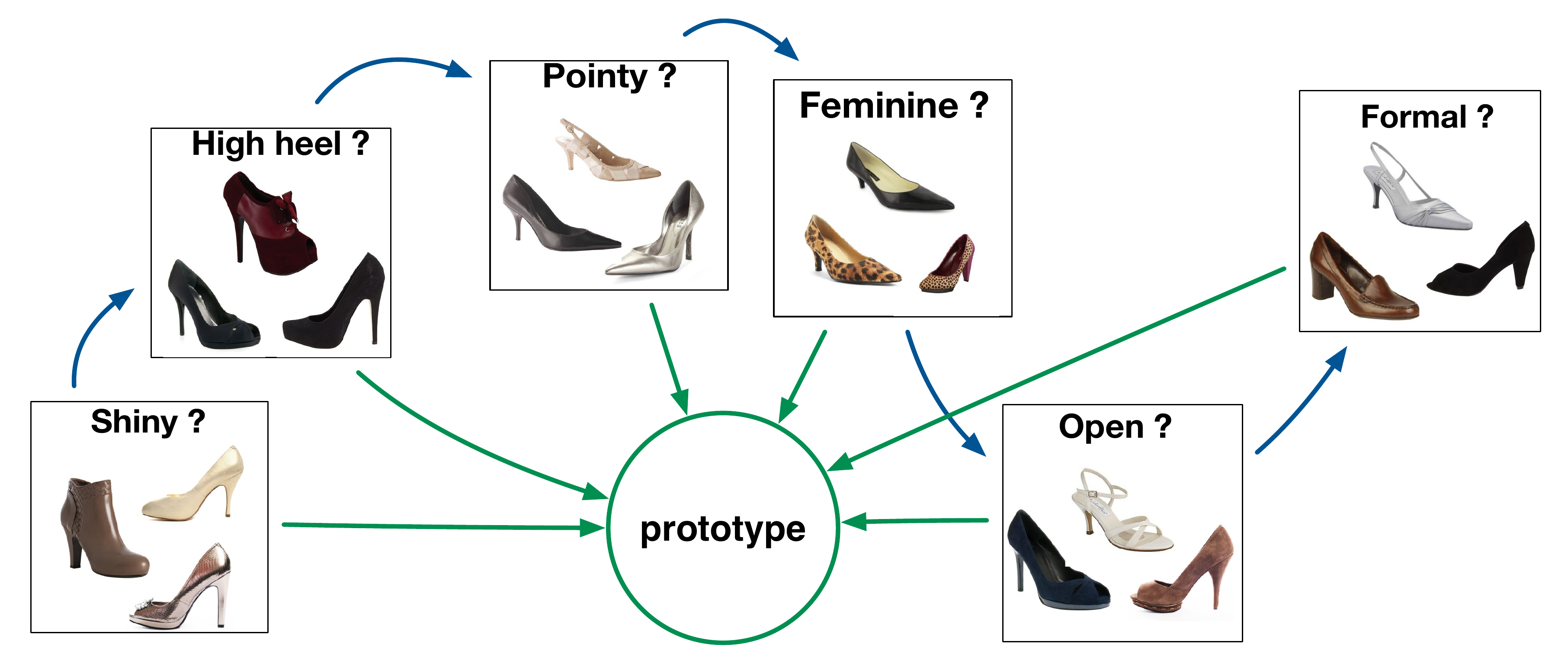}
\caption{Schematic illustration of the proposed multi-task learning approach.
If each task is related to some other task but not equally much to all others, 
learning tasks in a sequence (blue arrows) can be beneficial to classical multi-task 
learning based on sharing information from a single prototype (green arrows). }

\label{fig:motivation}
\end{figure}

The idea of regularizing by Euclidean distance between the weight vectors of different tasks is also commonly used in domain adaptation scenario where the learner has access to two or more prediction tasks but is interested in performing well only on one of them. All other tasks serve only as sources of additional information.
This setup has been shown to lead to effective algorithms in various computer vision applications: object detection~\cite{Aytar}, personalized image search~\cite{DBLP:conf/iccv/KovashkaG13a}, hand prosthetics~\cite{DBLP:conf/icra/OrabonaCCFS09} and image categorization~\cite{Tommasi_BMVC_2009,DBLP:conf/cvpr/TommasiOC10}. 
Though the domain adaptation scenario is noticeably different from the multi-task one, as it concentrates on solving only one task instead of 
all of them, the two research areas are closely related in term of their
methodology and therefore can benefit from each other.
In particular, the learning algorithm we propose can be seen as a way to decompose a multi-task problem into a set of domain adaptation problems.

Our approach is motivated by the human educational process.
If we consider students at school, they, similarly to a multi-task learner, are supposed to learn many concepts.
However, they learn them not all simultaneously, but in a sequence.
By processing tasks in a meaningful order, students are able to gradually increase their knowledge and reuse previously accumulated information to learn new concepts more effectively.
Inspired by this example we propose to solve tasks in a sequential manner by transferring information from a previously learned task to the next one instead of solving all of them simultaneously.
This approach makes learning more flexible in terms of variability between the tasks and memory efficient as it does not require processing all training data at the same time.

As for students at school, the order in which tasks are solved may crucially affect the overall performance of the learner.
We study this question by using PAC-Bayesian theory~\cite{McAllester} to prove a generalization bound that depends on the data representation and algorithm used to solve the tasks.
The bound quantifies the effectiveness of the order in which tasks are solved and therefore can be used to find a beneficial order.
Based on the bound we propose a theoretically justified algorithm that automatically chooses a favourable sequence for learning.
Our experimental results on two real-world image datasets show that learning tasks in a sequence can be superior to independent learning as well as to the standard multi-task approach of solving them jointly, and that our algorithm is able to reliably discover an advantageous order.

\section{Related Work}
While our work is based on the idea of transferring information through weight vectors,  other approaches to multi-task learning have been proposed as well.  
A popular idea in the machine learning literature is that parameters of related tasks can be represented as linear combinations of a small number of common latent basis vectors.
Argyriou \emph{et al.} proposed a method to learn such representations using sparsity regularization in~\cite{Argyriou}. 
This method was later extended to allow partial overlap between groups of tasks in~\cite{Kumar}.
It was also adapted to the lifelong setting in~\cite{Ruvolo}, where 
Ruvolo and Eaton proposed a way to sequentially update the model 
as new tasks arrive.
In~\cite{Ruvolo2013Active}, the same authors further extended it to the case when the learner is allowed to choose which task to solve next and they proposed
using different heuristics for making this choice. 
Experimentally, subspace-based methods have shown good performance in situations where many tasks are available and the underlying feature representations are low-dimensional.
When the feature dimensionality gets larger, however, their computational cost grows quickly, and this makes them not applicable for the type of computer vision problems we are interested in.\footnote{\scriptsize{In preliminary experiments we tried to use ELLA~\cite{Ruvolo}, as one of the fastest exiting methods, but found the experiments intractable to do at full size. A simplified setup produced results clearly below that of other baselines.}}
An exception is~\cite{dinesh-cvpr2014}, where Jayaraman \emph{et al.} apply subspace-based method to jointly learn multiple attribute predictors. 
However, even there, dimensionality reduction was required. 

Methods based on the sharing of weight vector have also been 
generalized since their original introduction in~\cite{Evgeniou}, 
in particular to relax the assumption that all tasks have to 
be related.
In~\cite{Evgeniou:2005:LMT:1046920.1088693}, Evgeniou \emph{et al.} 
achieved this by introducing a graph regularization. 
Alternatively, Chen \emph{et al.}~\cite{DBLP:journals/corr/abs-1005-3579}  
proposed to penalize deviations in weight vectors for highly correlated tasks. 
However, these methods require prior knowledge about the amount 
of similarities between tasks. 
In contrast, the algorithm we present in this work does not assume all tasks to be related, yet does not need a priori information regarding their similarities, either. 

The question how to order a sequence of learning steps to achieve best performance has previously been studied mainly in the context of single task learning, where the question is in which order one should process the training examples. 
In~\cite{Ben09} Bengio \emph{et al.} showed experimentally that choosing training examples in an order of gradually increasing difficulty can lead to faster training and higher prediction quality. 
Similarly, Kumar \emph{et al.}~\cite{Kum10} introduced the self-paced learning algorithm, which automatically chooses the order in which training examples are processed for solving a non-convex learning problem.
In the context of learning multiple tasks, the question in which order to learn them was introduced in~\cite{lad2009toward}, where Lad \emph{et al.} proposed an algorithm for optimizing the task order based on pairwise preferences.
However, they considered only the setting in which tasks are performed in a sequence through user interaction and therefore their approach is not applicable in the standard multi-task scenario.
In the setting of multi-label classification, the idea of decomposing a multi-target problem into a sequence of single-target ones was proposed by Read \emph{et al.} in~\cite{Read:2009:CCM:1617459.1617477}. However, there the sharing of information occurs through augmentations of the feature vectors,
not through a regularization term. 

\section{Method}

In the multi-task scenario a learning system observes multiple supervised learning tasks, 
for example, recognizing objects or predicting attributes.
Its goal is to solve all these tasks by sharing information between them.
Formally we assume that the learner observes $n$ tasks, denoted by $t_1,
\dots,t_n$, which share the same input and output spaces, $\mathcal{X}\subset\mathbb{R}^d$ and $\mathcal{Y}=\{-1,+1\}$, respectively.  
Each task $t_i$ is defined by the corresponding set $S_i=\{(x^i_1,y^i_1),\dots,(x^i_{m_i},y^i_{m_i})\}$ of $m_i$ training points sampled i.i.d. according to some unknown distribution $D_i$ over $\mathcal{X}\times\mathcal{Y}$.  
We also assume that for solving each task the learner uses a linear predictor $f(x)=\sign\langle w,x\rangle$, where $w\in\mathbb{R}^d$ is a weight vector, and we measure the classification performance by the $0/1$ loss, $l(y_1,y_2)=\llbracket y_1\neq y_2\rrbracket$.
The goal of the learner is to find $n$ weight vectors $w_1,\dots,w_n$ such that the average expected error on tasks $t_1,\dots,t_n$ (given that the predictions are made by the corresponding linear predictors) is minimized:
\begin{equation}
\er(w_1,\dots,w_n) = \frac{1}{n}\sum_{i=1}^n\underset{(x,y)\sim D_i}{\mathbf{E}}\llbracket y\neq \sign\langle w_i,x\rangle\rrbracket.
\label{exp_error}
\end{equation}

\subsection{Learning in a fixed order}
We propose to decompose a multi-task problem of solving $n$ tasks into $n$ domain adaptation problems.
Specifically, we assume that the tasks $t_1,\dots,t_n$ are processed sequentially in some order $\pi\in\mathcal{S}_n$, where $\mathcal{S}_n$ is a symmetric group of all permutation over $n$ elements, and information is transferred between subsequent tasks: from $t_{\pi(i-1)}$ to $t_{\pi(i)}$ for all $i=2,\dots,n$.
In this procedure the previously solved task serves as a source of additional information 
for the next task and any of the existing domain adaptation methods can be used. 
In this paper we use an Adaptive SVM~\cite{Aytar} to train classifiers for every task due to its proved effectiveness in computer vision applications.
For a given weight vector $\tilde{w}$ and training data for a task, the Adaptive SVM performs the following optimization:
\begin{gather}
\label{ASVM}
\min_w\;\;\|w - \tilde{w}\|^2 
+ \frac{C}{m}\sum_{j=1}^{m}\xi_j\\
\notag
\text{sb.t.} \;\; y_j\langle w,x_j\rangle\geq 1-\xi_j, \;\; \xi_j\geq 0 \;\; \text{for all}\;\; 1\leq j\leq m.
\end{gather} 
Specifically, to learn a predictor for the task $t_{\pi(i)}$ we solve~\eqref{ASVM} using the weight vector obtained for the previous task, $w_{\pi(i-1)}$, as $\tilde{w}$.
For the very first task, $t_{\pi(0)}$, we use the standard linear SVM, i.e. $\tilde{w}=\textbf{0}$.
To simplify the notation, we define $\pi(0)$ to be $0$ and $w_0$ to be the zero vector. 

Note that this approach does not rely on the assumption that all the tasks $t_1,\dots,t_n$ are equally related.
However its performance will depend on the order $\pi$ as it needs subsequent tasks to be related.
In the next section we study this question using statistical learning theory and introduce an algorithm for automatically defining a beneficial data-dependent order.

\subsection{Learning a data-dependent order}
\label{theory}

Here we examine the role of the order $\pi$ in terms of the average expected error~\eqref{exp_error} of the resulting solutions.
However, we do not limit our theoretical analysis to the case of using Adaptive SVMs as described earlier.
Specifically, we only assume that the learning algorithm used for solving each individual task $t_{\pi(i)}$ is the same for all tasks and deterministic.
This algorithm, $\mathcal{A}(w_{\pi(i-1)}, S_{\pi(i)})$, returns $w_{\pi(i)}$ based on the solution $w_{\pi(i-1)}$ obtained for a previously solved task and training data $S_{\pi(i)}$. 
The following theorem provides an upper-bound on the average expected error~\eqref{exp_error} of the obtained predictors (the proof can be found in the Appendix~\ref{ProofThm1}).

\begin{theorem}
For any deterministic learning algorithm $\mathcal{A}$ and any $\delta>0$, the following inequality holds with probability at least $1-\delta$ (over sampling the training sets $S_1,\dots,S_n$) uniformly for any order $\pi\in\mathcal{S}_n$:
\begin{align}
\notag
\frac{1}{2n}&\sum_{i=1}^n\underset{(x,y)\sim D_i}{\mathbf{E}}\llbracket y\neq\sign\langle w_i,x\rangle\rrbracket\leq\\
\notag
\frac{1}{n}&\sum_{i=1}^n\Bigg[\frac{1}{m_{\pi(i)}}\sum_{j=1}^{m_{\pi(i)}}\bar{\Phi}\left(\frac{y^{\pi(i)}_j\langle w_{\pi(i)},x^{\pi(i)}_j\rangle}{||x^{\pi(i)}_j||}\right)
 +\\
 &\frac{||w_{\pi(i)}\!-\!w_{\pi(i-1)}||^2}{2\sqrt{\bar{m}}}\Bigg]
 +\frac{1}{8\sqrt{\bar{m}}} - \frac{\log\delta}{n\sqrt{\bar{m}}}+\frac{\log n}{\sqrt{\bar{m}}},
\label{objective_function}
\end{align}
where $\bar{m}$ is the harmonic mean of the sample sizes $m_1,\dots,m_n$, $\bar{\Phi}(z)=\frac{1}{2}\left(1-\Erf\left(\frac{z}{\sqrt{2}}\right)\right)$, 
$\Erf(z)$ is the Gauss error function~\cite{Germain:2009, langford2002}, $\pi(0)=0$, $w_0=\mathbf{0}$ and $w_{\pi(i)}=\mathcal{A}(w_{\pi(i-1)},S_{\pi(i)})$.
\label{bound}
\end{theorem}

The left hand side of the inequality~\eqref{objective_function} is one half of the average expected error on tasks $t_1,\dots,t_n$.
This is the quantity of interest that the learner would like to minimize.
However, since the underlying data distributions $D_1,\dots,D_n$ are unknown, it is not computable.
In contrast, its upper bound given by the right hand side of~\eqref{objective_function} contains only computable quantities.
It is an average of $n$ terms (up to constants which do not depend on $\pi$), where each term corresponds to one task.
If we consider the term corresponding to the task $t_{\pi(i)}$, its first part is an analogue of the training error.
Each term  $\bar{\Phi}\left(y^{\pi(i)}_j\langle w_{\pi(i)},x^{\pi(i)}_j\rangle||x^{\pi(i)}_j||^{-1}\right)$ has a value between $0$ and $1$ and is a monotonically decreasing function of the distance between the training point $x^{\pi(i)}_k$ and the hyperplane defined by $w_{\pi(i)}$.
Specifically, it is close to $0$ when $x^{\pi(i)}_j$ is correctly classified and has large distance from the separating hyperplane, it is close to $1$ when the point is in the wrong halfspace far from the hyperplane and is $0.5$ when $x^{\pi(i)}_j$ lies on the hyperplane.
%
%
%
Therefore it captures the confidence of the predictions on the training set.
The second part of the term corresponding to the task $t_{\pi(i)}$ is a complexity term.
It measures the similarity between subsequent tasks $t_{\pi(i-1)}$ and $t_{\pi(i)}$ by the $L_2$-distance between the obtained weight vectors.
As a result the value of the right-hand side of~\eqref{objective_function} depends on $\pi$ and captures the influence that the task $t_{\pi(i)}$ may have on the subsequent tasks $t_{\pi(i+1)},\dots,t_{\pi(n)}$.
Therefore it can be seen as a quality measure of order $\pi$: a low value of the right hand side of~\eqref{objective_function} ensures a low expected error~\eqref{exp_error}.
It leads to an algorithm for obtaining an order $\pi$ that is adjusted to the tasks $t_1,\dots,t_n$  by minimizing the right hand side of \eqref{objective_function} based on the data $S_1,\dots,S_n$. 
Because \eqref{objective_function} holds uniformly in $\pi$, its guarantees also hold for the learned order\footnote{\scriptsize{Note that, in contrast, algorithm $\mathcal{A}$ is assumed to be fixed in advance.
Therefore the order $\pi$ is the only parameter that can be adjusted by minimizing~\eqref{objective_function} with preservation of the performance guarantees given by Theorem~\ref{bound}.}}.
%


Minimizing the right hand side of~\eqref{objective_function} is an expensive combinatorial problem, because it requires searching over all possible permutations $\pi\in\mathcal{S}_n$.
We propose an incremental procedure for performing this search approximately.
We successively determine $\pi(i)$ by minimizing the corresponding term of the upper bound~\eqref{objective_function} with respect to yet unsolved tasks.
Specifically, at the $i$-th step, when $\pi(1),\dots,\pi(i-1)$ are already defined, we search for a task $t_k$ that minimizes the following objective function and is not included in the order $\pi$ yet:
\begin{equation}
\frac{1}{m_k}\sum_{j=1}^{m_k}\bar{\Phi}\left(\frac{y^k_j\langle w_k,x^k_j\rangle}{||x^k_j||}\right)+\frac{||w_k-w_{\pi(i-1)}||^2}{2\sqrt{\bar{m}}},
\label{step2}
\end{equation}
where $w_k=\mathcal{A}(w_{\pi(i-1)}, S_k)$. 
We let $\pi(i)$ be the index of the task that minimizes~\eqref{step2}.
Suchwise at every step we choose the task that is easy (has low empirical error) and similar to the previous one (the corresponding weight vectors are close in terms of $L_2$ norm).
Therefore this optimization process well fits humans intuitive concept of starting with the simplest task and proceeding with most similar ones.
The resulting procedure in the case of using Adaptive SVM~\eqref{ASVM} for solving every task is summarized in Algorithm~\ref{alg1} and we refer to it as SeqMT.

\begin{algorithm}[]
\caption{Sequential Learning of Multiple Tasks}\label{alg1}
\begin{algorithmic}[1]
\STATE {\bfseries Input} $S_1,\dots,S_n$ \COMMENT{training sets}
\STATE $\pi(0)\gets 0$, $w_0\gets\boldsymbol{0}$
\STATE $T\gets\{1,2,\dots,n\}$ \COMMENT{indices of yet unused tasks} 
\FOR {$i = 1$ to $n$}
\FORALL {$k\in T$}
\STATE $w_k\gets$ solution of~\eqref{ASVM} using $S_k$, $w_{\pi(i-1)}$
\ENDFOR
\STATE $\pi(i)\gets$ minimizer of~\eqref{step2} w.r.t. $k$
\STATE $w_{\pi(i)}\gets$ $w_k$ where $k={\pi(i)}$
\STATE $T\gets T\setminus\{\pi(i)\}$
\ENDFOR
\STATE {\bfseries Return} $w_{1},\dots,w_{n}$ and $\pi(1),\dots,\pi(n)$ 
\end{algorithmic}
\end{algorithm}

\subsection{Learning with multiple subsequences}

The proposed algorithm, SeqMT, relies on the idea that all tasks can be ordered in a sequence, where each task is related to the previous one.
In practice, this is not always the case, since we can have outlier tasks that are not related to any other tasks, or we can have several groups of tasks, in which case it is beneficial to form subsequences within the groups, but it is disadvantageous to join them into one single sequence. 
%

Therefore, we propose an extension of the SeqMT model, that allows tasks to form subsequences, where the information is transferred only between the tasks within the subsequence.
Our multiple subsequences version, MultiSeqMT, also chooses tasks iteratively, but at any stage it allows the learner to choose whether to continue one of the existing subsequences or to start a new one. 
In order to decide which task to solve next and which subsequence to continue with it, the learner performs a two-stage optimization.
First, for each of the exiting subsequences $s$ (including empty one that corresponds to the no transfer case) the learner finds the task $t_s$ that is the most promising to continue with.
This is done in the same way as how the next task is chosen in the SeqMT algorithm.
Afterwards, the learner compares the values of criterion~\eqref{step2} for every pair $(s,t_s)$ and chooses the subsequence $s^*$ with the minimal value and continues it with the task $t_{s^*}$.      
%
%
Please, refer to the Appendix~\ref{MSubSeq} for exact formulation.
\section{Experiments}
In this section we verify our two main claims: 
1) learning multiple tasks in a sequential manner can be more effective than learning them jointly; 
2) we can find \emph{automatically} a favourable order in terms of average classification accuracy. 
We use two publicly available datasets:  \textit{Animals with Attributes (AwA)}\footnote{\scriptsize{\url{http://attributes.kyb.tuebingen.mpg.de/}}}~\cite{Lam13} and \textit{Shoes}\footnote{\scriptsize{\url{http://tamaraberg.com/attributesDataset/index.html}}}~\cite{Berg2010} augmented with attributes\footnote{\scriptsize{\url{http://vision.cs.utexas.edu/whittlesearch/}}}~\cite{whittlesearch}. 
In the first experiment, we study the case when each task has a certain level of difficulty for learning the object class, which is defined by human annotation in a range from easiest to hardest. 
We show the advantage of a sequential learning model over learning multiple tasks jointly and learning each task independently. 
We also study the automatically determined orders in more detail, comparing them with the orders when learning goes from easiest to hardest tasks in the spirit of human learning. 
In the second experiment, we study the scenario of learning visual attributes 
that characterize shoes across different shoe models. 
In this setting, some tasks are clearly related such as \emph{high heel} and \emph{shiny}, and some tasks are not, such as \emph{high heel} and \emph{sporty}. 
Therefore, we also apply the variant of our algorithm that allows multiple subsequences, 
showing that it better captures the task structure and is therefore the favourable learning strategy. 

\subsection{Learning the order of easy and hard tasks} 
\begin{figure*}[th]
\centering
\includegraphics[scale=0.85]{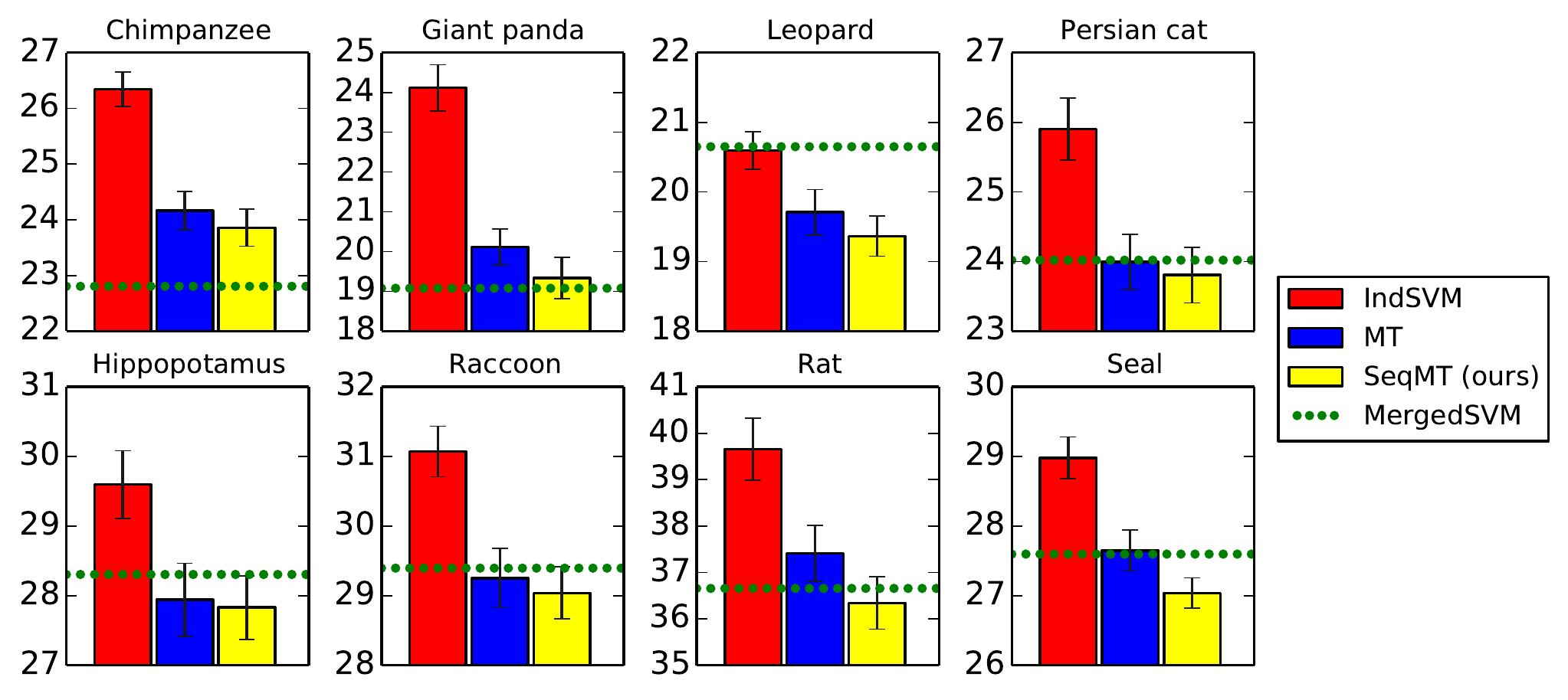}
\caption{Learning the order of easy and hard tasks on AwA dataset: comparison of the proposed SeqMT method with the multi-task (MT) and the single-task (IndSVM) baselines. 
The height of the bar corresponds to the average error rate performance over 5 tasks across 20 repeats (the lower the better). 
As a reference, we also provide the MergedSVM baseline, trained on data that is merged from all tasks. 
For a complete table with all the results, please refer to the Appendix~\ref{Exp}. 
}
\label{fig:easyhard_mt}
\end{figure*}

We focus on eight classes from the \emph{AwA} dataset: \emph{chimpanzee, giant panda, leopard, persian cat, hippopotamus, raccoon, rat, seal}, for which human annotation is available, whether
an object is \emph{easy} or \emph{hard} to recognize in an image~\cite{ShaQuaLam14}. 
For each class the annotation specifies ranking scores of its images from easiest to hardest. 
To create easy-hard tasks, we split the data in each class into five equal parts with respect to their easy-hard ranking and use these parts to create five tasks per class. 
Each part has on average $120$ samples except the class \emph{rat}, for which AwA contains 
few images, so there are only approximately $60$ samples per part.
Each task is a binary one-versus-rest classification of one of the parts against the remaining seven classes. 
For each task we balance $21$ vs $21$ training images and $77$ vs $77$ test images ($35$ vs $35$ in case of class \emph{rat}) with equal amount of samples from each of the classes acting as negative examples. The data between different tasks does not overlap. 
As our feature representation, we use $2000$ dimensional bag-of-words histograms obtained from SURF descriptors \cite{Bay08} provided together with the dataset. 
We $L_2$-normalize the features and augment them with a unit element to act as a bias term. 

{\bf Evaluation metric.} To evaluate the performance of the methods we use the classification error rate. 
We repeat each of the experiments $20$ times with different random data splits and measure the average error rate across the tasks.
%
%
We report mean and standard error of the mean of this value over all repeats.  
\begin{figure*}[th!!]
\centering
\includegraphics[scale=0.7]{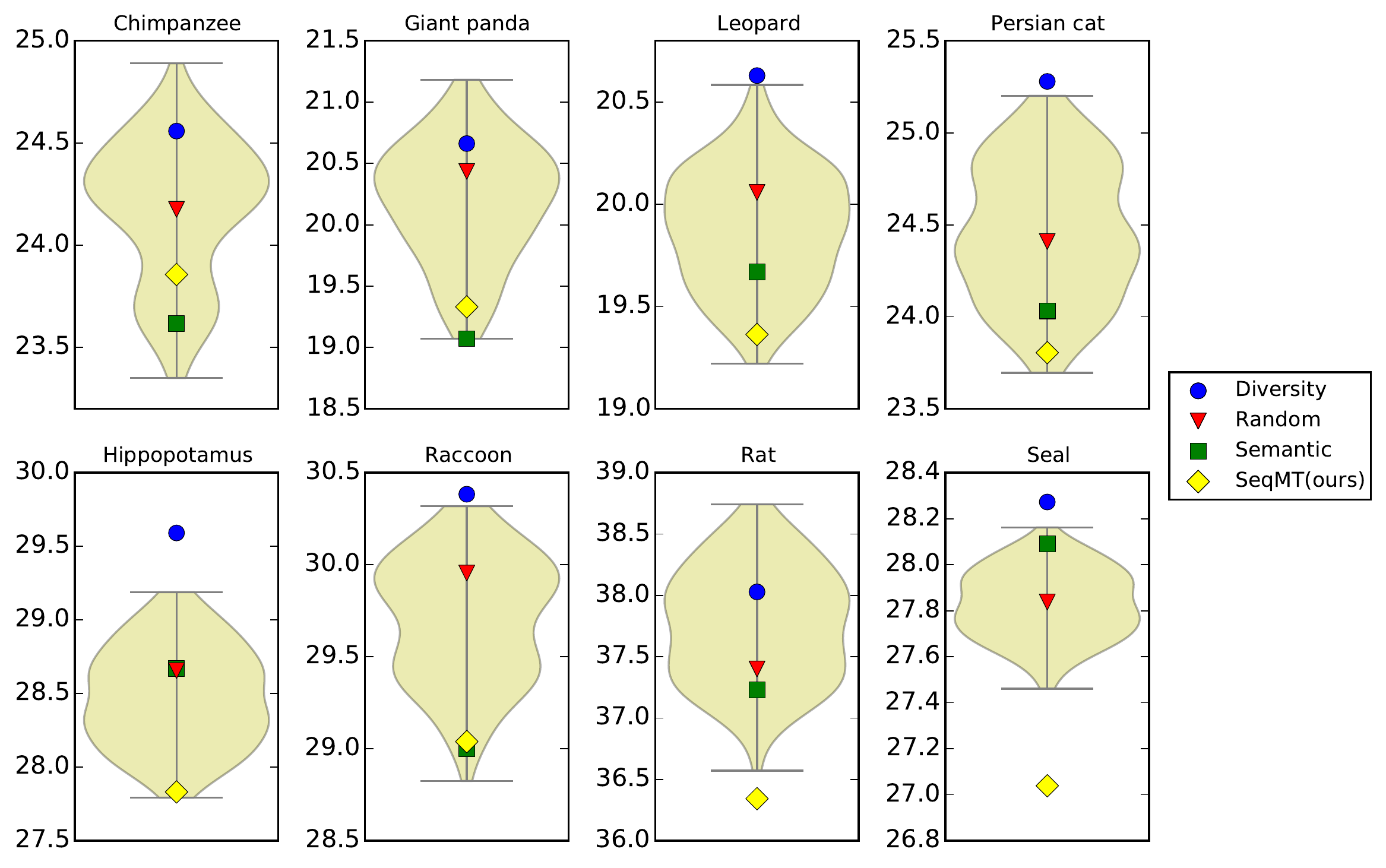}
\caption{Study of different task order strategies in the experiment with \textit{AwA} dataset. 
Four main baselines SeqMT, Semantic, Random, Diversity have a distinctive marker and color, and their vertical location captures averaged error rate performance (shown on the vertical axis). 
The performance of all possible orders is visualized as a background violin plot, where one horizontal slice of the shaded area reflects how many different orders achieve this error rate performance. 
Note, symmetry of this area is used only for aesthetic purposes.
For a complete table with all the results, please refer to the Appendix~\ref{Exp}. Best viewed in colors.
}
\label{fig:easyhard_seq}
\end{figure*}
\begin{table*}[th!!]
\centering
\scalebox{0.8}{
\begin{tabular}{c|l|c|c|c|c|c|c|c|c|}		
		&Chimpanzee&	Giant panda&	Leopard&	Persian cat&	Hippopotamus&	Raccoon&	Rat&	Seal\\
\hline
Error+Compl	&$\bf{23.86\pm0.33}$	&$\bf{19.33\pm0.52}$	&$\bf{19.36\pm0.29}$	&$\bf{23.81\pm0.40}$	&$\bf{27.83\pm0.46}$	&$\bf{29.04\pm0.37}$		&$\bf{36.34\pm0.57}$	&$\bf{27.04\pm0.22}$\\ 
Error		&$24.47\pm0.42$		&$20.02\pm0.58$		&$19.97\pm0.27$		&$24.84\pm0.46$		&$29.07\pm0.55$		&$29.75\pm0.31$		&$38.00\pm0.54$		&$28.27\pm0.38$\\ 
Compl		&$23.94\pm0.32$		&$19.44\pm0.50$		&$\bf{19.36\pm0.29}$	&$\bf{23.81\pm0.40}$	&$\bf{27.83\pm0.46}$	&$\bf{29.04\pm0.37}$		&$\bf{36.34\pm0.57}$	&$\bf{27.04\pm0.22}$\\ 
\hline
\end{tabular} }
\caption{Trade-off between complexity and error terms in the proposed SeqMT strategy of choosing next task~\eqref{step2} on the AwA dataset. 
The numbers are average error rate performance over 5 tasks across 20 repeats.  
}
\label{tab:SeqMT_Err_Reg}
\end{table*}

{\bf Baselines.}
We compare our sequential learning model (SeqMT) with the multi-task algorithm from~\cite{Evgeniou}, \cite{Finkel2009} that treats all tasks symmetrically (MT). 
Specifically, MT regularizes the weight vectors for all tasks to be similar to a prototype $w_0$ that is learned jointly with the task weight vectors by solving the following optimization problem:
\begin{align}
\notag
&\min_{w_0, w_i, \xi^i_j} \|w_0\|^2+\frac{1}{n}\sum_{i=1}^n\|w_i-w_0\|^2 + \frac{C}{n}\sum_{i=1}^n\frac{1}{m_i}\sum_{j=1}^{m_i}\xi^i_j\\
&\text{subject to}\;\; y^i_j\langle w_i,x^i_j\rangle\geq 1- \xi^i_j, \;\;\ \xi^i_j\geq 0 \;\; \ \text{for all} \; i,j.
\label{MT}
\end{align}
In order to study how relevant the knowledge transfer actually is, we compare SeqMT with a linear SVM baseline that solves each task independently (IndSVM). 
As a reference, we also provide the performance of a linear SVM trained on data that is merged from all tasks (MergedSVM). 

To understand the impact that the task order has on the classification accuracy   
we compare the performance of SeqMT with baselines that learn tasks in random order (Random), and in order from easiest to hardest (Semantic) according to the human annotation as if it was given to us. 
Another baseline we found related is inspired by the diversity heuristic from~\cite{Ruvolo2013Active}. It defines the next task to be solved by maximizing~\eqref{step2} instead of minimizing it. We refer to it as Diversity. 
%

{\bf Model selection.} We perform a cross validation model selection approach for choosing the regularization trade-off parameter $C$ for each of the methods. 
In all our experiment, we select $C$ over $8$ parameter values $\{10^{-2}, 10^{-1}\dots,10^{5}\}$ using $5\times5$ fold cross-validation. 
%

{\bf Results.}
We present the results of this experiment in Figure~\ref{fig:easyhard_mt} and Figure~\ref{fig:easyhard_seq}. 
As we can see from Figure~\ref{fig:easyhard_mt}, the proposed SeqMT method outperforms MT and IndSVM algorithms in all $8$ cases.
This shows that knowledge transfer between the tasks is clearly advantageous
in this scenario, 
and it supports our claims that learning tasks sequentially is more effective than learning them jointly if not all tasks are equally related. 
%
As expected, the reference baseline MergedSVM improves over single-task baseline IndSVM in all but one case, as training with more data has better generalization ability. 
%
In some cases, the MergedSVM performs on par or even better than SeqMT and MT methods, as for example, in cases of \emph{chimpanzee} and \emph{giant panda}. 
We expect that this happens when tasks are so similar that a single hyperplane can explain most of them. 
In this case, MergedSVM benefits from the amount of data that is available to find this hyperplane. 
When the tasks are different enough, MergedSVM is unable to explain all of them with one shared hyperplane and loses to SeqMT and MT models, that learn one hyperplane per task.
This can be see, e.g. in the cases of \emph{hippopotamus} and \emph{seal}, and particularly 
much in the case of \emph{leopard}, where the MergedSVM does not improve even over independent training.
%

\begin{figure}[b]
\centering
\includegraphics[scale=0.45]{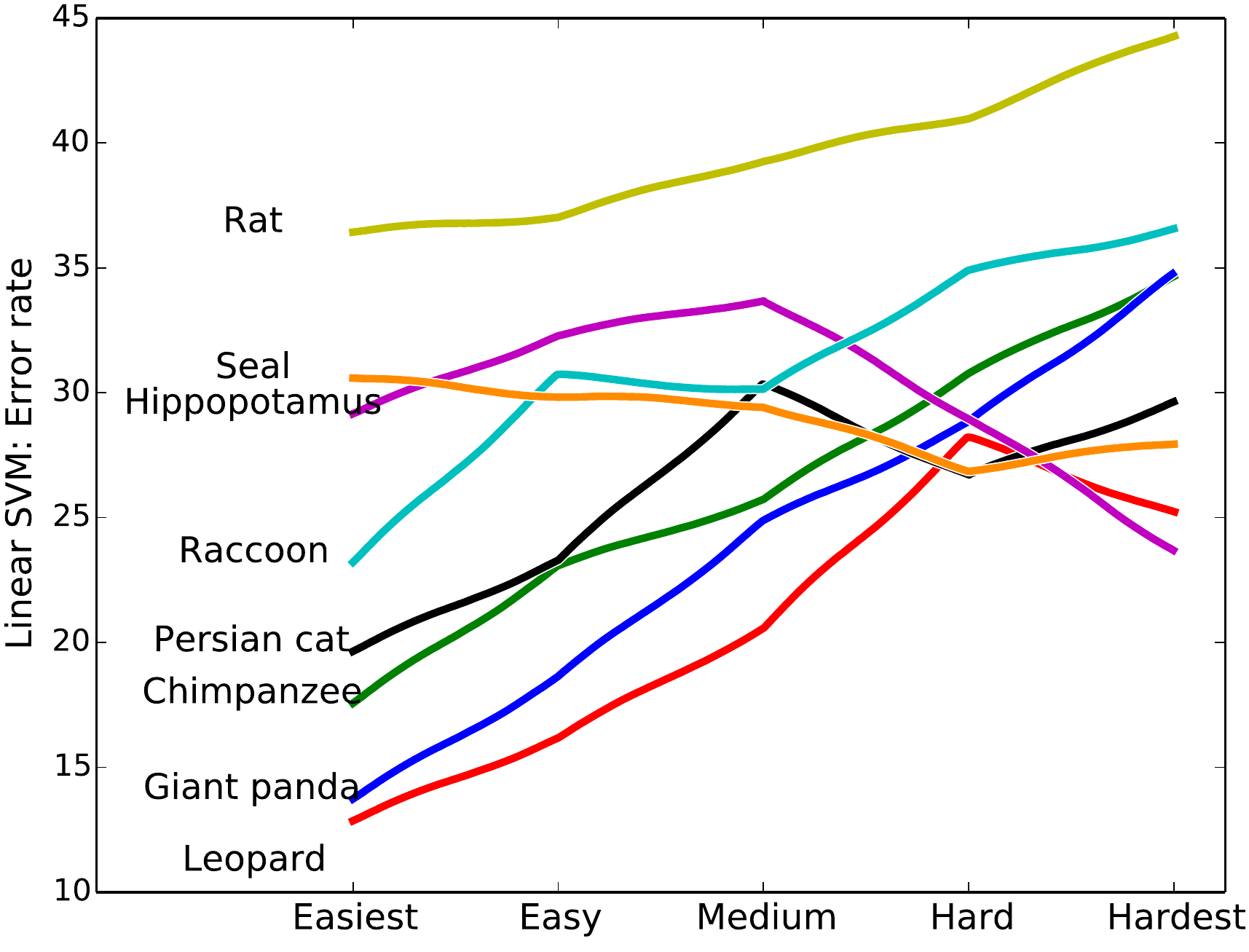}
\caption{Visualization of machine learning performance (linear SVM) w.r.t. human annotation of easy-hard tasks for AwA dataset experiment. 
Ideally, when human and machine understanding coincide, it would be a bottom-up diagonal line.  
Therefore, in cases of \emph{seal} and \emph{hippopotamus} we would not expect that learning in semantic order would lead us to the best performing strategy. Best viewed in colors.
}
\label{fig:human_machine}
\end{figure}

Next we examine the importance of the order in which the tasks are being solved,  reporting our findings in Figure~\ref{fig:easyhard_seq}. 
All methods in this study use a sequence of Adaptive SVMs as a learning algorithm for solving the next task and differ only by how the order of tasks is defined. 
%
In all $8$ cases the proposed SeqMT algorithm outperforms the Random order baseline, which learns the tasks in a random order\footnote{\scriptsize{A different random order is taken for each class for each of the $20$ repeats.}}. 
%
The Diversity algorithm is much worse than other baselines, presumably because the max heuristic of choosing the next task is not effective in this setting. 
As a reference, we also check the Semantic baseline when the tasks are being solved from easiest to hardest (as if we had prior information about the easy-hard order of the tasks\footnote{\scriptsize{This order is fixed for each class for each of the $20$ repeats.}}).  
In $6$ out of $8$ classes, the order learned by our SeqMT model (yellow rhombus) is better or on par with the Semantic (green square), except for classes \emph{chimpanzee} and \emph{giant panda}, where we did not manage to learn the best order. 
Interestingly, for some classes following the strategy of semantic order is worse or on par with learning them in a random order (cases with \emph{seal} and \emph{hippopotamus}).
We credit this to the fact that human perception of easiness and hardness does not always coincide with what is easy and hard to learn for a machine learning algorithm. In fact, in cases of \emph{seal} and \emph{hippopotamus}, the human and machine understanding are rather opposite: the hardest task for human is the easiest from machine learning perspective, and the easiest task for human is hardest or medium hard for the learning algorithm. 
Hence, learning these classes in random order leads to better results than learning in a fixed unfavourable order.
We check this by computing the error rates of single SVMs trained per each task: easiest, easy, medium, hard and hardest as defined by human studies and visualize the results in Figure~\ref{fig:human_machine}. 

Finally, for each class we compute the performance of all possible orders to learn $5$ tasks, which result in $120$ baselines\footnote{\scriptsize{One baseline defines one fixed order across all $20$ repeats. In SeqMT, we learn an adaptive order that can differ across the repeats.}}. 
We visualize the performance of all orders as a violin plot~\cite{violin}, where one horizontal slice of the shaded area reflects how many different orders achieve this error rate (performance stated on the vertical axis).
Overall, SeqMT is highly competitive with best possible fixed orders, clearly outperforming them in two cases of \emph{rat} and \emph{seal} (rhombus is lower than the yellow area), and loosing in \emph{chimpanzee}, which we have observed before. 
Thus, learning the adaptive order of tasks based on the training sets is advantageous to solving them in a fixed order. 

We also study the importance of the two terms in the objective function~\eqref{step2} for choosing the next task. 
For this, we compare our algorithm to two simplifications: choosing the next task based on the training error only (Error) and choosing the next task based on the complexity term only (Compl). 
The results in Table~\ref{tab:SeqMT_Err_Reg} suggest that the complexity term, i.e. the similarity 
between tasks, is the more important component, but that its combination with the error term 
achieves never worse and sometimes even better results. 

To conclude, our proposed algorithm orders the tasks into a learning sequence to achieve the best performance results, and is beneficial to all other strategies including the order annotated for human learning. 

\subsection{Learning subsequences of related attributes} 
We focus on $10$ attributes that describe shoe models~\cite{whittlesearch}: \emph{pointy at the front, open, bright in color, covered with ornaments, shiny, high at the heel, long on the leg, formal, sporty, feminine} and $10$ classes from the \emph{Shoes} dataset: \emph{athletic, boots, clogs, flats, heels, pumps, rain boots, sneakers, stiletto, wedding shoes}. 
%
Attribute description comes in form of class ranking from $1$ to $10$, with $10$ denoting class that ``has it the most'' and $1$ denoting class that ``has it the least''.
We form $10$ binary classification tasks, one for each attribute, 
using samples from top-2 classes as positive (classes with $10$ and $9$ ranks) and samples from bottom-2 classes as negative (classes with $1$ and $2$ ranks). For more clarifications on attribute-class description, see the Appendix~\ref{Exp}.
For each task we balance $50$ vs $50$ training images and $300$ vs $300$ test images, randomly sampled from each class in equal amount. 
The data between different tasks does not overlap. 
As feature representation, we use $960$ dimensional GIST descriptor concatenated with $L_1$-normalized $30$ dimensional color descriptor, 
augmented with a unit element as bias term. 

{\bf Baselines.}
In addition to all baselines described in the previous section, we add the MultiSeqMT method that allows to learn multiple subsequences of attributes (with the information transfer within a subsequence). Additionally we include a baseline RandomMultiSeq that learns attributes in random order with an option to randomly start a new subsequence. 

\begin{table}[ht]
\centering
\begin{tabular}{c|l|c|c}
& Methods  	& Average error & ~\\ \hline
&IndSVM		&$10.34\pm0.13$ \\
&MergedSVM	&$29.67\pm0.10$\\
&MT		&$10.37\pm0.13$\\
&SeqMT (ours)		&$10.96\pm0.12$\\
&MultiSeqMT (ours)	&$\bf{9.95\pm0.12}$\\\hline\hline
&Diversity	&$12.66\pm0.17$\\
&Random		&$12.14\pm0.20$\\
&RandomMultiSeq	&$10.89\pm0.14$\\
\end{tabular} 
\caption{Learning subsequences of related attributes on Shoes dataset. 
We compare the proposed MultiSeqMT and SeqMT methods with the multi-task (MT) and the single-task (IndSVM) baselines, and report the MergedSVM result as a reference baseline. 
We examine the importance of subsequences in which the tasks are being solved and compare our methods with Diversity, Random and RandomMultiSeq baselines. 
The numbers correspond to average error rate performance over $10$ tasks across $20$ repeats (the lower the better). 
The best result is highlighted in {\bf boldface}.
}
\label{tab:shoes}
\end{table}

{\bf Results.}
We present the main results of this experiment in Table~\ref{tab:shoes}. 
As we can see from it, the proposed MultiSeqMT method outperforms all other baselines and is a favourable strategy in this scenario. 
It is better than the SeqMT model which confirms that learning multiple subsequences is advantageous, when not all given tasks are equally related. 
The single-task learning baseline IndSVM is rather strong and performs on par with the Multi-task learning MT baseline, possibly because multi-task learning is negatively affected by it forcing transfer between unrelated tasks.
As expected, MergedSVM is unable to explain all tasks with one hyperplane and performs very poorly in this case. 

Similarly to the previous experiment, we examine the importance of sequences and subsequences in which the tasks are being solved. 
First, we compare the performance of the MultiSeqMT and SeqMT methods with the baselines that learn tasks in certain order (last three rows in the Table~\ref{tab:shoes}), and then we will share our findings about the learned subsequences of attributes. 

As we can see from Table~\ref{tab:shoes}, MultiSeqMT is able to order the tasks into subsequences in the most effective way. 
Learning multiple random subsequences as RandomMultiSeq does is better than learning a single sequence of all tasks, as SeqMT, Random and Diversity baselines do. 
However since SeqMT performs on par with RandomMultiSeq and clearly better than Random baseline, we conclude, that even with one sequence we are able to learn a good order of tasks that is discretely affected by transfer between unrelated tasks. 
The Diversity baseline is worse than other baselines also in this setting. 

Finally, we analyze the subsequences that MultiSeqMT has learned, finding some relatively 
stable patterns across the repeats. 
There are six attributes, \emph{shiny, high at the heel, pointy at front, feminine, open} and \emph{formal}, that can benefit from each other and often form a subsequence of related tasks. 
Inside the group, the attributes \emph{shiny} and \emph{high at the heel} frequently start the subsequence and transfer happens between both of interchangeably. 
The next attributes that often follow the previous two are \emph{pointy at front} and \emph{feminine}; they are also closely related and interchangeable in order.
The attribute \emph{open} is not always in the subsequence, but once it is included, this attribute transfers to \emph{formal}, which often ends the subsequence.

The remaining four attributes, \emph{bright in color, covered with ornaments, long on the leg} and \emph{sporty}, either form smaller subsequences, sometimes of two tasks only, or they appear as separate tasks. 
Occasionally there is transfer from \emph{long on the leg} attribute to \emph{covered with ornaments}, which 
we credit to the fact the shoe class \emph{boots} shares a high rank for both of those attributes. 
In half of the cases, the attributes \emph{sporty} and \emph{bright in color} are not related to the other tasks and form their own subsequences.

\section{Conclusion}

In this work, we proposed to solve multiple tasks in a sequential manner and studied the question if and how the order in which a learner solves a set of tasks influences its overall performance.
%
%
%
First, we provide a theoretical result: a generalization bound that can be used to access the quality of the learning order. 
Secondly, we proposed a principled algorithm for choosing an advantageous order based on the  theoretical result.
Finally, we tested our algorithm on two datasets and showed that: 1) learning multiple tasks sequentially can be more effective than learning them jointly; 2) the order in which tasks are solved effects the overall classification performance; 
3) our method is able to automatically discover a beneficial order.
%
%
%

A limitation of our model is that currently it allows to transfer only from the previous task to solve the current one, hence it outputs a sequence of related tasks or multiple task subsequences. 
In future work, we plan to extend our model by relaxing this condition and allowing the tasks to be organized in a tree, or a more general graph structure.

\appendix
\section{Proof of Theorem 1}
\label{ProofThm1}
We apply PAC-Bayesian theory to prove a generalization bound for the case of sequential task solving. For more details on it see \cite{catoni2007pac, Langford05, Seeger}.

Assume that the learner observes a sequence of tasks in a fixed order, $t_1,...,t_n$, with corresponding training sets, $S_1,...,S_n$, where $S_i=\{(x^i_1,y^i_1),...,(x^i_{m_i},y^i_{m_i})\}$ consists of $m_i$ i.i.d. samples from a task-specific data distribution $D_i$.
We assume that all tasks share the same input set $\mathcal{X}$ and output set  $\mathcal{Y}$ and that the learner uses the same loss function $l:\mathcal{Y}\times\mathcal{Y}\rightarrow[0,1]$ and hypothesis set $H\subset\{h:\mathcal{X}\rightarrow\mathcal{Y}\}$ for solving these tasks.
The learner solves only one task at a time by using some arbitrary but fixed deterministic algorithm $\mathcal{A}$ that produces a posterior distribution $Q_i$ over $H$ based on training data $S_i$ and some prior knowledge $P_i$, which is also expressed in form of probability distribution over the hypothesis set.
Moreover, we assume that the solution $Q_i$ plays the role of a prior for the next task, i.e. $P_{i+1}=Q_i$ ($P_1$ is just some fixed distribution, $Q_0$). 
For making predictions for task $t_i$ the learner uses the Gibbs predictor, associated with the corresponding posterior distribution $Q_i$.
For an input $x\in\mathcal{X}$ this randomized predictor samples $h\in H$ according to $Q_i$ and returns $h(x)$.
The goal of the learner is to perform well on all tasks, $t_1,...,t_n$, i.e. to minimize the average expected error of the Gibbs classifiers defined by $Q_1,\dots,Q_n$:
\begin{align}
\notag
\er = \frac{1}{n}\sum_{i=1}^n \er_i(Q_i(Q_{i-1}, S_i))=\\
\frac{1}{n}\sum_{i=1}^n\mathbf{E}_{(x,y)\sim D_i}\mathbf{E}_{h\sim Q_i}l(h(x),y).
\label{def:expected_error}
\end{align}
Since the data distributions of the tasks $t_1,...,t_n$ are unknown, one can not directly compute \eqref{def:expected_error}. 
However, it can be approximated by the empirical error based on the observed data:
\begin{align}
\notag
\widehat{\er} = \frac{1}{n}\sum_{i=1}^n \widehat{\er}_i(Q_i(Q_{i-1}, S_i))=\\
\frac{1}{n}\sum_{i=1}^n\frac{1}{m_i}\sum_{j=1}^{m_i}\mathbf{E}_{h\sim Q_i}l(h(x^i_i),y^i_j).
\label{def:empirical_error}
\end{align}
The following theorem provides an upper bound on the difference between the two quantities~\eqref{def:expected_error} and~\eqref{def:empirical_error}:
\begin{theorem}
For any fixed distribution $Q_0$, learning algorithm $\mathcal{A}$ and any $\delta>0$  the following inequality holds with probability at least $1-\delta$ (over sampling the training sets $S_1,...,S_n$):
\begin{align}
\notag
\er\leq\eer+&
\frac{1}{n\sqrt{\bar m}}\KL\big(Q_1\times\cdots\times Q_n||Q_0\times\cdots\times Q_{n-1}\big)\\
+&\frac{1}{8\sqrt{\bar m}} - \frac{\log\delta}{n\sqrt{\bar m}},
\label{thm1_equation}
\end{align}

where $Q_i=\mathcal{A}(Q_{i-1}, S_i)$ is a posterior distribution for the task $t_i$ learned by $\mathcal{A}$ based on $Q_{i-1}$ and $S_i$, $\bar{m}=\left(\frac{1}{n}\sum_{i=1}^n\frac{1}{m_i}\right)^{-1}$ is the harmonic mean of the sample sizes and $\KL$ denotes Kullback-Leibler divergence.
\label{theorem1}
\end{theorem}

\begin{proof}
First we use Donsker-Varadhan's variational formula \cite{Seldin12} to change the expectation over posteriors $(Q_1,...,Q_n)$ to the expectation over priors $(Q_0,Q_1,...,Q_{n-1})$:
\begin{align}
\notag
\er-\eer \leq  \ \frac{1}{\lambda}\Big(\KL\big(Q_1\times\cdots\times Q_n||Q_0\times\cdots\times Q_{n-1}\big)  
\\
+ \log\underset{h_1\sim Q_0}{\mathbf{E}}...\underset{h_n\sim Q_{n-1}}{\mathbf{E}}\exp\Big(\frac{\lambda}{n}\sum_{i=1}^n(\er_i(h_i)-\eer_i(h_i))\Big)\Big),
\label{KL-inequality}
\end{align}
where $\er_i(h)$ is the expected loss of a hypothesis $h$ computed with respect to the data distribution of task $t_i$ and $\eer_i(h)$ is the corresponding empirical loss, computed on $S_i$.
This inequality holds for any $\lambda>0$.

Note, that $Q_i$ may depend on $S_1,...,S_i$, but does not depend on $S_{i+1},...,S_n$. Therefore:
\begin{align}
\notag
\underset{S_1\cdots S_n}{\mathbf{E}}&\underset{h_1\sim Q_0}{\mathbf{E}}\!...\!\underset{h_n\sim Q_{n-1}}{\mathbf{E}}\!\exp\!\left(\!\frac{\lambda}{n}\sum_{i=1}^n(\er_i(h_i)-\eer_i(h_i))\!\right)\!=\\
\notag
&\underset{h_n\sim Q_{0}}{\mathbf{E}}\underset{S_1}{\mathbf{E}}\exp\left(\frac{\lambda}{n}(\er_1(h_1)-\eer_1(h_1))\right)\cdots\\
&\underset{h_n\sim Q_{n-1}}{\mathbf{E}}\underset{S_n}{\mathbf{E}}\exp\left(\frac{\lambda}{n}(\er_n(h_n)-\eer_n(h_n))\right).
\label{long_expectations}
\end{align}

We fix $h_n\in H$. Then we can rewrite the last term of \eqref{long_expectations} in the following way:
\begin{align}
\notag
\exp\left(\frac{\lambda}{n}(\er_n(h_n)-\widehat{\er}_n(h_n))\right)=\\
\prod_{j=1}^{m_n}\exp\left(\frac{\lambda}{nm_n}\left(\er_n(h_n)-l(h_n(x^n_{j}),y^n_{j})\right)\right).
\end{align}
Since the data points in $S_n$ are i.i.d., all terms in this product are independent and take values between $\frac{\lambda(\er_n(h_n)-1)}{nm_n}$ and $\frac{\lambda \er_n(h_n)}{nm_n}$. %
Therefore, by Hoeffding's lemma \cite{Hoeffding:1963}, we obtain that the last term of \eqref{long_expectations} is bounded by a constant:
\begin{equation}
\notag
\underset{h_n\sim Q_{n-1}}{\mathbf{E}}\!\underset{S_n}{\mathbf{E}}\!\exp\!\left(\!\frac{\lambda}{n}(\er_n(h_n)\!-\!\eer_n(h_n))\!\right)\leq\exp\left(\frac{\lambda^2}{8n^2m_n}\right)
\end{equation}

We repeat the same procedure for all other tasks and obtain that:
\begin{align}
\notag
\underset{S_1...S_n}{\mathbf{E}}\underset{h_1\sim Q_0}{\mathbf{E}}\!...\!\underset{h_n\sim Q_{n-1}}{\mathbf{E}}\!&\exp\!\left(\!\frac{\lambda}{n}\sum_{i=1}^n(er_i(h_i)\!-\!\eer_i(h_i))\!\right)\!\leq\\
&\exp\left(\frac{\lambda^2}{8n\bar{m}}\right),
\end{align}
where $\bar{m}=\left(\frac{1}{n}\sum_{i=1}^n\frac{1}{m_i}\right)^{-1}$. 
Therefore, by Markov's inequality, with probability at least $1-\delta$:
\begin{align}
\notag
\underset{h_1\sim Q_0}{\mathbf{E}}...\underset{h_n\sim Q_{n-1}}{\mathbf{E}}&\exp\left(\frac{\lambda}{n}\sum_{i=1}^n(er_i(h_i)-\eer_i(h_i))\right)\leq\\
\frac{1}{\delta}&\exp\left(\frac{\lambda^2}{8n\bar{m}}\right).
\label{Markovs_inequality}
\end{align}

By combining \eqref{Markovs_inequality} with \eqref{KL-inequality} we get:
\begin{align}
\notag
\er\leq\eer +& \frac{1}{\lambda}\KL\big(Q_1\times\cdots\times Q_n||Q_0\times\cdots\times Q_{n-1}\big)\\
 +& \frac{\lambda}{8n\bar{m}}-\frac{1}{\lambda}\log\delta.
\end{align}
By setting $\lambda=n\sqrt{\bar{m}}$ we obtain the final result.
\end{proof}

Theorem~\ref{theorem1} holds only for tasks that are given to the learner in an arbitrary but fixed order, which must be chosen before observing the sample sets $S_1,\dots,S_n$. 
We can, however, extend it to hold uniformly for all orders of 
tasks: 
for each possible task order, $\pi\in\mathcal{S}_n$, where $\mathcal{S}_n$
is the \emph{symmetric group}, we use~\eqref{thm1_equation} 
with confidence parameter $\delta/n!$. We then combine 
all inequalities (of which there are $n!$ many) using the union bound, 
thereby obtaining the following generalization:
\begin{theorem}
For any fixed distribution $Q_0$, any learning algorithm $\mathcal{A}$ and any $\delta>0$ with probability at least $1-\delta$ (over sampling the training sets $S_1,...,S_n$) the following inequality holds uniformly for any order $\pi\in \mathcal{S}_n$: 

\begin{align}\label{thm2_equation}
&\er\leq\eer
+\frac{1}{8\sqrt{\bar{m}}}+\frac{\log n}{\sqrt{\bar{m}}}-\frac{\log\delta}{n\sqrt{\bar{m}}}+
\\
\notag
\frac{1}{n\sqrt{\bar{m}}}&\KL\big(Q_{\pi(1)}\times\cdots\times Q_{\pi(n)}||Q_0\times\cdots\times Q_{\pi(n-1)}\big),
\end{align}
where $Q_{\pi(i)}=\mathcal{A}(Q_{\pi(i-1)}, S_{\pi(i)})$, $\bar{m}=\left(\frac{1}{n}\sum_{i=1}^n\frac{1}{m_i}\right)^{-1}$ and $\pi(0)=0$.
\label{theorem2}
\end{theorem}

Theorem 1 is an instantiation of Theorem~\ref{theorem2} for the special case of binary classification using linear predictors.
Assume $\mathcal{Y}=\{+1,-1\}$, $\mathcal{X}\in \mathbb{R}^d$, and let $H$ be a set of linear predictors $\{\sign\langle w,x\rangle\}$, where $w\in\mathbb{R}^d$ is a weight vector. 
We also assume that the learner uses $0/1$ loss, $l(y_1,y_2)=\llbracket y_1\neq y_2\rrbracket$.
In this case the expected error of the Gibbs predictor is at least half the expected error of the corresponding majority vote predictor \cite{mcallester2003simplified}.
Therefore, by multiplying the right hand side of \eqref{thm2_equation} by a factor of $2$, one obtains a generalization bound for deterministic majority vote classifier.

The case of linear predictors can be captured by the PAC-Bayesian setting if prior and posterior distributions are Gaussian \cite{Herbrich}.
More formally, assume that $Q_i=\mathcal{N}(w_i, \Id)$ for $i=0,...,n$, i.e. Gaussian distributions with unit variance that differ only by the value of their mean vectors.   
Due to the symmetry of the Gaussian distribution, the predictor defined by $w_i$ is equivalent to the majority vote predictor corresponding to distribution $Q_i$.
Hence one can use the result of Theorem~\ref{theorem2} in the case of deterministic linear predictors.
We also assume that the learner uses an algorithm, $\mathcal{A}$, that for every task $t_i$ returns $w_i$ based on the mean vector of the used prior distribution and training data $S_i$. 

By computing the complexity term from \eqref{thm2_equation} we obtain:
\begin{align}
\notag
\KL(Q_{\pi(1)}\!\times\!\cdots\!\times\! Q_{\pi(n)}||Q_0\!\times\!\cdots\!\times\! Q_{\pi(n-1)})\!=\\
\!\sum_{i=1}^{n}\!\KL(Q_{\pi(i)}||Q_{\pi(i-1)})\! =\! \sum_{i=1}^n\frac{||w_{\pi(i)}-w_{\pi(i-1)}||^2}{2},
\label{KL}
\end{align}
where $\pi(0)=0$, $w_0=\textbf{0}$ and $w_{\pi(i)}=\mathcal{A}(w_{\pi(i-1)}, S_{\pi(i)})$. 
Note that the loss of the Gibbs classifier defined by $Q_i$ on a point $(x,y)$ is given by $\bar{\Phi}\Big(\frac{yx^Tw_i}{||x||}\Big)$, where $\bar{\Phi}(z)=\frac{1}{2}\left(1-\Erf\left(\frac{z}{\sqrt{2}}\right)\right)$ and $\Erf(z)=\frac{2}{\sqrt{\pi}}\int_0^ze^{-t^2}dt$ is the Gauss error function~\cite{Germain:2009, langford2002}.
Together with \eqref{KL} it gives us the result of Theorem 1.

\section{Additional information for MultiSeqMT}
\label{MSubSeq}
Assume that, as in the case of learning in a fixed order described in Theorem~\eqref{theorem1}, $n$ tasks $t_1,...,t_n$ are processed one after another from $t_1$ till $t_n$.
We extend the sequential learning scenario by allowing the learner to not transfer information between some of the subsequent task.
Specifically, if the posterior distribution $Q_{i}$ obtained for task $t_{i}$ is not informative with respect to the next task, $t_{i+1}$, the learner may use original, fixed distribution $Q_0$ as a prior for $t_{i+1}$ instead of $Q_{i}$.
Such scenario can be described by introducing the set of flags $b_i\in\{0,1\}$ for $i=2,...,n$, where $b_i=1$ means that information from task $t_{i-1}$ is transferred to the task $t_{i}$, in other words $Q_{i-1}$ is used as a prior for solving $t_{i}$, while $b_i=0$ denotes that there is no transfer from $t_{i-1}$ to $t_i$ and $Q_0$ is used as a prior $P_i$.

In the same manner, as we proved Theorem~\eqref{theorem1}, we can prove the following generalization bound for the case of sequential learning with ability to not transfer information between subsequent tasks:

\begin{theorem}
For any fixed distribution $Q_0$, set of flags $b_i\in\{0,1\}$ for $i=2,...,n$, learning algorithm $\mathcal{A}$ and any $\delta>0$  the following inequality holds with probability at least $1-\delta$ (over sampling the training sets $S_1,...,S_n$):
\begin{align}
\notag
\er\leq\eer+&
\frac{1}{n\sqrt{\bar m}}\KL\big(Q_1\times\cdots\times Q_n||P_1\times\cdots\times P_n\big)\\
+&\frac{1}{8\sqrt{\bar m}} - \frac{\log\delta}{n\sqrt{\bar m}},
\label{thm1_equation}
\end{align}

where:
\begin{eqnarray*}
P_i &=& \begin{cases} Q_0 &\mbox{if } i=1 \;\;\mbox{or}\;\; b_i=0 \\ 
Q_{i-1} & \mbox{if } b_i=1 \end{cases}\\
Q_i&=&\mathcal{A}(P_i, S_i)\\
\bar{m}&=&\left(\frac{1}{n}\sum_{i=1}^n\frac{1}{m_i}\right)^{-1}.
\end{eqnarray*}
\label{theorem1_ext}
\end{theorem}

The result of Theorem~\eqref{theorem1_ext} holds for any, but fixed in advance order of tasks and set of flags $b_i$.
Now, we can extend it to hold uniformly for all possible partitions of tasks in subsequences and orders of tasks in each group.
First, note that there are $n!\leq n^n$ possible full orderings of $n$ tasks.
Second, there are $2^{n-1}$ possible ways to define flags $b_i$ for each task.
Therefore there are less than $n^n2^{n-1}$ possible partitions of tasks and groups and orderings inside each group.
We now let the confidence parameter to be $\delta/((2n)^n)$ and combine inequalities for all possible partitions and orderings (of which there are less than $(2n)^n$ many) using the union bound argument.
Thereby we obtain the following result:

\begin{theorem}
For any fixed distribution $Q_0$, learning algorithm $\mathcal{A}$ and any $\delta>0$ with probability at least $1-\delta$ (over sampling the training sets $S_1,...,S_n$) the following inequality holds uniformly for all orders $\pi\in\mathcal{S}$ and all set of flags $\{b_2,...,b_n\}\in\{0,1\}^{n-1}$:
\begin{align}\label{thm2_equation}
&\er\leq\eer
+\frac{1}{8\sqrt{\bar{m}}}+\frac{\log 2n}{\sqrt{\bar{m}}}-\frac{\log\delta}{n\sqrt{\bar{m}}}+
\\
\notag
\frac{1}{n\sqrt{\bar{m}}}&\KL\big(Q_{\pi(1)}\times\cdots\times Q_{\pi(n)}||P_{\pi(1)}\times\cdots\times P_{\pi(n)}\big),
\end{align}

where:
\begin{eqnarray*}
P_{\pi(i)} &=& \begin{cases} Q_0 &\mbox{if } i=1 \;\;\mbox{or}\;\; b_i=0 \\ 
Q_{\pi(i-1)} & \mbox{if } b_i=1 \end{cases}\\
Q_{\pi(i)}&=&\mathcal{A}(P_{\pi(i)}, S_{\pi(i)})\\
\bar{m}&=&\left(\frac{1}{n}\sum_{i=1}^n\frac{1}{m_i}\right)^{-1}.
\end{eqnarray*}
\label{theorem2_ext}
\end{theorem}

We can formulate the instantiation of Theorem~\eqref{theorem2_ext} for the case of linear predictors and $0/1$ loss using Gaussian distributions as we did for proving Theorem 1 based on Theorem~\eqref{theorem2}.
As a result, we obtain the following generalization bound:

\begin{theorem}
For any deterministic learning algorithm $\mathcal{A}$ and any $\delta>0$, the following holds with probability at least $1-\delta$ over sampling the training sets $S_1,...,S_n$ uniformly for any order $\pi$ in the symmetric group $\mathcal{S}_n$ and any set of flags $\{b_2,...,b_n\}\in\{0,1\}^{n-1}$:
\begin{align}
\notag
\frac{1}{2n}&\sum_{i=1}^n\underset{(x,y)\sim D_i}{\mathbf{E}}\llbracket y\neq\sign\langle w_i,x\rangle\rrbracket\leq\\
\notag
\frac{1}{n}&\sum_{i=1}^n\Bigg[\frac{1}{m_{\pi(i)}}\sum_{j=1}^{m_{\pi(i)}}\bar{\Phi}\left(\frac{y^{\pi(i)}_j\langle w_{\pi(i)},x^{\pi(i)}_j\rangle}{||x^{\pi(i)}_j||}\right)
 +\\
 &\frac{||w_{\pi(i)}\!-\!w_{\pi(i-1)}||^2}{2\sqrt{\bar{m}}}\Bigg]
 +\frac{1}{8\sqrt{\bar{m}}} - \frac{\log\delta}{n\sqrt{\bar{m}}}+\frac{\log 2n}{\sqrt{\bar{m}}},
\label{objective_function}
\end{align}
where:
\begin{eqnarray*}
w_{\pi(i)} &=& \begin{cases} \mathcal{A}(\mathbf{0}, S_{\pi(i)}) &\mbox{if } i=1 \;\;\mbox{or}\;\; b_i=0 \\ 
\mathcal{A}(w_{\pi(i-1)}, S_{\pi(i)}) & \mbox{otherwise} \end{cases}\\
\bar{\Phi}(z)&=&\frac{1}{2}\left(1-\Erf\left(\frac{z}{\sqrt{2}}\right)\right)\\
\Erf(z)&=&\frac{2}{\sqrt{\pi}}\int_0^ze^{-t^2}dt\\
\bar{m}&=&\left(\frac{1}{n}\sum_{i=1}^n\frac{1}{m_i}\right)^{-1}.
\end{eqnarray*} 

\label{bound_ext}
\end{theorem}

\begin{algorithm}[]
\caption{MultiSeqMT: Sequential Learning with Multiple Subsequences }\label{alg2}
\begin{algorithmic}[1]
\STATE {\bfseries Input} $S_1,\dots,S_n$ \COMMENT{training sets}
\STATE $T \gets \{1,2,\dots,n\}$ \COMMENT{indices of yet unused tasks}
\STATE $P \gets \{\boldsymbol{0}\}$ \COMMENT{$w$s of the last tasks in the existing subseq.}
\FOR {$i = 1$ to $n$}
\FORALL {$\tilde{w}\in P$}
\STATE $k(\tilde{w})\gets$ steps 5-8 of Algorithm 1 with\\ 
\pushcode[4] substituting $w_{\pi(i-1)}$ by $\tilde{w}$ in (4)
\ENDFOR
\STATE $w^*\gets$ minimizer of (4) w.r.t. $\tilde{w}$ with substituting \\ \pushcode[3] $w_{\pi(i-1)}$ by $\tilde{w}$ and $k$ by $k(\tilde{w})$ 
\STATE $w_{k(w^*)}\gets$ solution of (2) using $S_{k(w^*)}$ and \\
\pushcode[5] $w^*$ instead of $\tilde{w}$
\STATE $T\gets T\setminus\{k(w^*)\}$
\STATE $P \gets P\cup\{w_{k(w^*)}\}$
\IF {$w^* \neq \boldsymbol{0}$}
\STATE $P \gets P\setminus\{w^*\}$
\ENDIF
\ENDFOR
\STATE {\bfseries Return} $w_{1},\dots,w_{n}$
\end{algorithmic}
\end{algorithm}

\section{Additional information for experiments}
\label{Exp}
\begin{table*}[t!]
\centering
\scalebox{0.85}{
\begin{tabular}{l|c|c|c|c|c|c|c|c|}		
&Chimpanzee&	Giant panda&	Leopard&	Persian cat&	Hippopotamus&	Raccoon&	Rat&	Seal\\
\hline
IndSVM&		$26.34\pm0.31$	&$24.12\pm0.58$	&$20.60\pm0.27$	&$25.90\pm0.45$	&$29.60\pm0.49$	&$31.07\pm0.36$	&$39.66\pm0.66$	&$28.98\pm0.30$\\ 
MergedSVM&	$22.81\pm0.31$	&$19.08\pm0.47$	&$20.65\pm0.35$	&$24.02\pm0.44$	&$28.31\pm0.48$	&$29.40\pm0.57$	&$36.66\pm0.62$	&$27.60\pm0.33$\\ 
MT&		$24.16\pm0.35$	&$20.12\pm0.45$	&$19.71\pm0.33$	&$23.99\pm0.40$	&$27.94\pm0.52$	&$29.25\pm0.43$	&$37.41\pm0.60$	&$27.65\pm0.29$\\ 
SeqMT(ours)&	$23.86\pm0.33$	&$19.33\pm0.52$	&$19.36\pm0.29$	&$23.81\pm0.40$	&$27.83\pm0.46$	&$29.04\pm0.37$	&$36.34\pm0.57$	&$27.04\pm0.22$\\ 
Max&		$24.56\pm0.37$	&$20.66\pm0.49$	&$20.63\pm0.32$	&$25.28\pm0.34$	&$29.59\pm0.49$	&$30.38\pm0.51$	&$38.03\pm0.58$	&$28.27\pm0.37$\\ 
Error&		$24.47\pm0.42$	&$20.02\pm0.58$	&$19.97\pm0.27$	&$24.84\pm0.46$	&$29.07\pm0.55$	&$29.75\pm0.31$	&$38.00\pm0.54$	&$28.27\pm0.38$\\ 
Reg		&$23.94\pm0.32$	&$19.44\pm0.50$	&$19.36\pm0.29$	&$23.81\pm0.40$	&$27.83\pm0.46$	&$29.04\pm0.37$	&$36.34\pm0.57$	&$27.04\pm0.22$\\ 
Random		&$24.18\pm0.37$	&$20.44\pm0.46$	&$20.06\pm0.33$	&$24.41\pm0.37$	&$28.66\pm0.55$	&$29.95\pm0.48$	&$37.40\pm0.66$	&$27.84\pm0.27$\\ 
Semantic	&$23.62\pm0.32$	&$19.07\pm0.51$	&$19.67\pm0.30$	&$24.03\pm0.37$	&$28.67\pm0.47$	&$29.00\pm0.43$	&$37.23\pm0.54$	&$28.09\pm0.36$\\ 
\hline
Best	&$23.35\pm0.38$	&$19.07\pm0.51$	&$19.22\pm0.30$	&$23.69\pm0.46$	&$27.79\pm0.33$	&$28.82\pm0.46$	&$36.57\pm0.63$	&$27.46\pm0.37$\\
Worst	&$24.89\pm0.40$	&$21.18\pm0.48$	&$20.58\pm0.32$	&$25.20\pm0.39$	&$29.19\pm0.47$	&$30.32\pm0.51$	&$38.74\pm0.65$	&$28.16\pm0.28$\\ 
\end{tabular} }
\caption{Sequential learning of tasks from easiest to hardest in the AwA dataset. 
For each class and method, the numbers are average error rate and standard error of the mean over $20$ repeats.  
}
\label{tab:easyhard}
\end{table*}
\begin{table*}[t]
\centering
\scalebox{0.85}{
\begin{tabular}{|l|c|c|c|c|c|c|c|c|c|c|}
\hline	
 Attribute/Class & Athletic & Boots & Clogs & Flats & Heels & Pumps & Rain Boots & Sneakers & Stiletto & Wedding \\
\hline
 Pointy at the front & \cellcolor{blue!25}2 & 6 & 3 & 5 & \cellcolor{yellow!25}10 & \cellcolor{yellow!25}9 & 4 & \cellcolor{blue!25}1 & 8 & 7 \\
\hline 
 Open & 3 & \cellcolor{blue!25}2 & 8 & 5 & 7 & 6 & \cellcolor{blue!25}1 & 4 & \cellcolor{yellow!25}9 & \cellcolor{yellow!25}10\\
\hline 
 Bright in color & 6 & \cellcolor{blue!25}1 & \cellcolor{blue!25}2 & 8 & 4 & 3 & \cellcolor{yellow!25}10 & 7 & \cellcolor{yellow!25}9 & 5\\
\hline 
 Covered with ornaments & 4 & \cellcolor{yellow!25}9 & 6 & 5 & 8 & 7 & \cellcolor{blue!25}1 & 3 & \cellcolor{yellow!25}10 & \cellcolor{blue!25}2 \\
\hline 
 Shiny & \cellcolor{blue!25}2 & \cellcolor{yellow!25}9 & 4 & 3 & 6 & 5 & 8 & \cellcolor{blue!25}1 & \cellcolor{yellow!25}10 & 7 \\
\hline 
 High at the heel & 4 & 6 & 5 & \cellcolor{blue!25}1 & \cellcolor{yellow!25}9 & 8 & 3 & \cellcolor{blue!25}2 & \cellcolor{yellow!25}10 & 7\\
\hline 
 Long on the leg & 7 & \cellcolor{yellow!25}9 & \cellcolor{blue!25}2 & 3 & 6 & 5 & \cellcolor{yellow!25}10 & 8 & 4 & \cellcolor{blue!25}1 \\
\hline 
 Formal & 3 & 6 & 4 & 7& \cellcolor{yellow!25}9 & 8 & \cellcolor{blue!25}1 & \cellcolor{blue!25}2 & 5 & \cellcolor{yellow!25}10 \\
\hline 
 Sporty & \cellcolor{yellow!25}10 & 5 & 6 & 7 & 4 & 3 & 8 & \cellcolor{yellow!25}9 & \cellcolor{blue!25}1 & \cellcolor{blue!25}2 \\
\hline 
 Feminine & \cellcolor{blue!25}1 & 6 & 4 & 5 & \cellcolor{yellow!25}10 & \cellcolor{yellow!25}9 & 3 & \cellcolor{blue!25}2 & 8 & 7\\
\hline 
 \end{tabular} }
\caption{Ordering of classes with respect to attributes in the \textit{Shoes} dataset~\cite{whittlesearch}. Cells, coloured in \colorbox{blue!25}{blue}, represent classes that were used as negative examples and the ones coloured in \colorbox{yellow!25}{yellow} represent the ones used as positive examples for the corresponding attribute.
}
\label{tab:shoes_attr}
\end{table*}

\newpage
{\small
\bibliographystyle{ieee}
\bibliography{bibfile}
}

\end{document}